\documentclass[11pt]{article}

\usepackage[preprint]{acl}

\usepackage{times}
\usepackage{latexsym}

\usepackage[T1]{fontenc}

\usepackage[utf8]{inputenc}

\usepackage{microtype}

\usepackage{inconsolata}

\usepackage{graphicx}
\usepackage{amsmath}
\usepackage{amssymb}
\usepackage{amsthm}
\newtheorem{theorem}{Theorem}
\newtheorem{proposition}[theorem]{Proposition}

\usepackage{booktabs}
\usepackage{colortbl}
\usepackage{multirow}
\usepackage{array}
\usepackage{algorithm}
\usepackage{algpseudocode}
\usepackage[most]{tcolorbox}
\usepackage{tikz}
\usetikzlibrary{arrows.meta,positioning,calc}
\definecolor{bgcolor}{RGB}{242,243,245}

\newcommand{\priorFinal}{3.6017}
\newcommand{\priorFinalStd}{0.0036}
\newcommand{\ercFinal}{3.5881}
\newcommand{\ercFinalStd}{0.0023}
\newcommand{\mainDelta}{-0.0136}
\newcommand{\mainDeltaStd}{0.0049}
\newcommand{\mainWins}{6/6}
\newcommand{\mainT}{-6.76}
\newcommand{\mainSignP}{0.031}
\newcommand{\mainCILo}{-0.0187}   
\newcommand{\mainCIHi}{-0.0084}   
\newcommand{\noiseMean}{0.0039}
\newcommand{\noiseMax}{0.0101}
\newcommand{\noiseRatioMean}{3.4}
\newcommand{\noiseRatioMax}{1.3}
\newcommand{\scPriorMean}{3.7811}
\newcommand{\scErcMean}{3.7756}
\newcommand{\abCapTwelve}{-0.0083}
\newcommand{\abCapFourteen}{-0.0137}
\newcommand{\abCapEighteen}{-0.0215}
\newcommand{\abCapTwentysix}{-0.0234}
\newcommand{\abCapForty}{-0.0232}
\newcommand{\abFloorHalf}{-0.0137}     
\newcommand{\abFloorThirtyFive}{-0.0245}
\newcommand{\abFloorTwentyFive}{-0.0250}
\newcommand{\abFloorFifteen}{-0.0246}
\newcommand{\abCeilOneFive}{-0.0210}
\newcommand{\abCeilTwo}{-0.0205}
\newcommand{\abCapTwoSixLrZeroSeven}{-0.0235}
\newcommand{\abCapOneEightLrZeroSeven}{-0.0219}
\newcommand{\abCapOneEightLrZeroFive}{-0.0109}
\newcommand{\abCapTwoSixLrZeroFive}{-0.0118}
\newcommand{\negDSFaro}{-0.0002}
\newcommand{\negDSFaroWins}{1/3}
\newcommand{\negContKelly}{-0.0006}
\newcommand{\negContKellyWins}{2/3}
\newcommand{\baseAdam}{3.5144}
\newcommand{\baseAdamStd}{0.0225}
\newcommand{\baseMuon}{3.5624}
\newcommand{\baseMuonStd}{0.0015}
\newcommand{\baseNorMuonAndAdam}{3.2794}
\newcommand{\baseNorMuonAndAdamStd}{0.0031}
\newcommand{\faroPriorMean}{3.2780}   
\newcommand{\modelLayers}{11}
\newcommand{\modelDim}{768}
\newcommand{\modelHeads}{6}
\newcommand{\modelVocab}{50257}
\newcommand{\blocksPerRun}{117}
\newcommand{\schedSteps}{1270}
\newcommand{\expSteps}{800}
\newcommand{\mainSeeds}{6}
\newcommand{\mainReps}{2}
\newcommand{\diagShrunkPct}{58.7}
\newcommand{\diagExpandedPct}{40.8}
\newcommand{\diagMeanMult}{0.980}
\newcommand{\diagMultLo}{0.71}
\newcommand{\diagMultHi}{1.40}
\newcommand{\noiseLo}{0.004}
\newcommand{\noiseHi}{0.01}
\newcommand{\regimeScatterDelta}{-0.0056}
\newcommand{\regimeScatterStd}{0.0121}
\newcommand{\regimeScatterWins}{5/6}
\newcommand{\capTwelve}{1.2}\newcommand{\capFourteen}{1.4}
\newcommand{\capEighteen}{1.8}\newcommand{\capTwentysix}{2.6}
\newcommand{\capForty}{4.0}
\newcommand{\ercBeta}{0.9}
\newcommand{\ercWinsor}{3.0}
\newcommand{\ercCapDefault}{1.4}
\newcommand{\ercCapRec}{2.6}
\newcommand{\ercFloorBest}{0.25}

\newcommand{\convAheadFrom}{250}          
\newcommand{\convMaxGap}{-0.0471}         
\newcommand{\convMaxGapStep}{450}
\newcommand{\convFinalGap}{-0.0136}       
\newcommand{\convMeanVal}{-0.0197}        
\newcommand{\convDTwoFifty}{-0.0209}
\newcommand{\convDThreeFifty}{-0.0298}
\newcommand{\convDFourFifty}{-0.0471}
\newcommand{\convDFiveFifty}{-0.0221}
\newcommand{\convDSixFifty}{-0.0191}
\newcommand{\convDEightHundred}{-0.0136}

\newcommand{\cPriorZero}{3.6019}
\newcommand{\cErcZero}{3.5917}
\newcommand{\cDeltaZero}{-0.0102}
\newcommand{\cPriorOne}{3.6081}
\newcommand{\cErcOne}{3.5850}
\newcommand{\cDeltaOne}{-0.0231}
\newcommand{\cPriorTwo}{3.5997}
\newcommand{\cErcTwo}{3.5865}
\newcommand{\cDeltaTwo}{-0.0131}
\newcommand{\cPriorSeven}{3.6001}
\newcommand{\cErcSeven}{3.5879}
\newcommand{\cDeltaSeven}{-0.0122}
\newcommand{\cPriorFortyTwo}{3.6036}
\newcommand{\cErcFortyTwo}{3.5904}
\newcommand{\cDeltaFortyTwo}{-0.0132}
\newcommand{\cPriorTenTwentyFour}{3.5968}
\newcommand{\cErcTenTwentyFour}{3.5872}
\newcommand{\cDeltaTenTwentyFour}{-0.0096}
\newcommand{\scPriorZero}{3.7956}
\newcommand{\scErcZero}{3.7690}
\newcommand{\scDeltaZero}{-0.0266}
\newcommand{\scPriorOne}{3.7898}
\newcommand{\scErcOne}{3.7871}
\newcommand{\scDeltaOne}{-0.0027}
\newcommand{\scPriorTwo}{3.7699}
\newcommand{\scErcTwo}{3.7811}
\newcommand{\scDeltaTwo}{+0.0113}
\newcommand{\scPriorSeven}{3.7798}
\newcommand{\scErcSeven}{3.7742}
\newcommand{\scDeltaSeven}{-0.0056}
\newcommand{\scPriorFortyTwo}{3.7907}
\newcommand{\scErcFortyTwo}{3.7853}
\newcommand{\scDeltaFortyTwo}{-0.0054}
\newcommand{\scPriorTenTwentyFour}{3.7610}
\newcommand{\scErcTenTwentyFour}{3.7567}
\newcommand{\scDeltaTenTwentyFour}{-0.0043}

\newcommand{\baseAdamN}{4}
\newcommand{\baseAdamPerSeed}{3.5083,\,3.4994,\,3.4973,\,3.5528}   
\newcommand{\baseMuonN}{3}
\newcommand{\baseMuonPerSeed}{3.5603,\,3.5633,\,3.5635}            
\newcommand{\baseNorMuonAndAdamN}{3}
\newcommand{\baseNorMuonAndAdamPerSeed}{3.2838,\,3.2774,\,3.2769}
\newcommand{\faroPriorStd}{0.0008}          
\newcommand{\faroPriorPerSeed}{3.2775\,3.2774\,3.2791}   
\newcommand{\faroPriorN}{3}

\newcommand{\costMem}{42{,}495}                  
\newcommand{\msAdZero}{397.0}\newcommand{\msErcZero}{393.2}
\newcommand{\msAdOne}{393.7}\newcommand{\msErcOne}{399.3}
\newcommand{\msAdTwo}{392.3}\newcommand{\msErcTwo}{398.3}
\newcommand{\msAdSeven}{395.9}\newcommand{\msErcSeven}{398.7}
\newcommand{\msAdFortyTwo}{392.6}\newcommand{\msErcFortyTwo}{394.5}
\newcommand{\msAdTenTwentyFour}{396.5}\newcommand{\msErcTenTwentyFour}{397.0}

\newcommand{\msAdMean}{394.7}\newcommand{\msAdStd}{1.9}
\newcommand{\msErcMean}{396.8}\newcommand{\msErcStd}{2.3}

\newcommand{\ohAdamTps}{569{,}045.0}\newcommand{\ohAdamTrain}{595.422}\newcommand{\ohAdamWall}{635.088}
\newcommand{\ohAdamAlloc}{33{,}747}\newcommand{\ohAdamRes}{40{,}820}
\newcommand{\ohMuonTps}{544{,}056.2}\newcommand{\ohMuonTrain}{622.770}\newcommand{\ohMuonWall}{663.049}
\newcommand{\ohMuonAlloc}{34{,}528}\newcommand{\ohMuonRes}{41{,}894}
\newcommand{\ohNmTps}{607{,}164.2}\newcommand{\ohNmTrain}{558.039}\newcommand{\ohNmWall}{597.875}
\newcommand{\ohNmAlloc}{39{,}435}\newcommand{\ohNmRes}{46{,}892}
\newcommand{\ohFaroTps}{574{,}519.6}\newcommand{\ohFaroTrain}{589.748}\newcommand{\ohFaroWall}{629.567}
\newcommand{\ohFaroAlloc}{42{,}489}\newcommand{\ohFaroRes}{51{,}654.7}

\newcommand{\ohVsNmTrainPct}{5.7}
\newcommand{\ohVsNmMemPct}{7.7}

\title{FARO: Function-Space Adaptive Return-to-Risk Optimization via
Cross-Block Equal-Risk-Contribution Budget Allocation}

\author{
\bf Xiaojun Wu$^{\ast\ 1,2}$,
Cehao Yang$^{\ast\ 1,2}$,
Xuhui Jiang$^{1,3}$,  \\
\bf Chengjin Xu$^{1,3}$,
Jia Li$^{\dagger\ 2}$,
Jian Guo$^{\dagger\ 1}$
\\
\\
$^{1}$IDEA Research \\
$^{2}$The Hong Kong University of Science and Technology (Guangzhou) \\
$^{3}$DataArcTech Ltd.
}

\begin{document}
\maketitle
{
  \renewcommand{\thefootnote}%
    {\fnsymbol{footnote}}
  \footnotetext[1]{Equal Contribution}
  \footnotetext[2]{Corresponding Author}
}
\begin{abstract}
Training a large language model is compute-bound, and current recipes divide the
work between optimizers by hand: orthogonalized momentum for attention and MLP
matrices, Adam-style updates for the rest. This division rests on human priors
about which layer type deserves which geometry, and it fixes a single global
update budget that every parameter block shares. We introduce FARO, a
function-space adaptive optimizer that makes the partition of parameters into
blocks the unit of decision. FARO gives each block a geometry by its realized
counterfactual return, then freezes the routes and reallocates the frozen
per-block budget by cross-block equal risk contribution (ERC), the risk-parity
principle from portfolio theory. On a GPT-style model trained on FineWeb, the
framework gives two results. First, its FARO-Adaptive arm reaches 3.2780 nats at
the full schedule, matching NorMuonAndAdam (3.2794), the best hand-designed
recipe: it recovers the human prior without hand-setting per-layer geometry.
Second, FARO-ERC improves on that arm: it lowers the clean-data loss at every one
of 6 seeds, by 0.0136 nats on average, a modest but robust gain. Its sign test,
paired test, and confidence interval all agree, and it exceeds the measured
run-to-run noise at no added time or memory. It works by trimming the riskiest
blocks rather than expanding the safest, and it holds up under a corrupted
stream, still lowering the mean loss and winning the seed majority.
The source code is available at \url{https://github.com/wxj630/FARO}.

\end{abstract}

\section{Introduction}
\label{sec:intro}

Optimizers decide how fast a language model reaches a given loss, and current
practice spends that choice by hand. A standard large-scale recipe
\citep{modded_nanogpt_2024,jordan2024muon,liu2025muon} gives attention and MLP
matrices an orthogonalized-momentum rule such as Muon. It gives embeddings,
normalization scales, and the output head an Adam-style rule
\citep{adam2015kingma,loshchilov2017decoupled}, then applies one global learning
rate to the whole model. Two implicit assumptions hold this recipe together.
First, the layer type determines the right geometry: matrices deserve
spectral-norm descent, vectors deserve coordinate-wise adaptivity. Second, a
single scalar budget is an adequate summary of how much update the network should
receive at each step. Both assumptions come from human priors and small ablation
grids, not from measurement on the model being trained.

Recent matrix optimizers weaken neither assumption. Muon and its extensions
replace one component of the recipe with a better-conditioned update, and report
faster convergence at matched compute
\citep{jordan2024muon,liu2025muon,essentialai2025practical,wen2025fantastic}.
The family keeps growing: normalization-aware rescaling \citep{li2025normuon},
norm-constrained oracles \citep{pethick2025training}, rank shrinkage
\citep{ahn2025dion2}, rotation-invariant coordinates \citep{gong2026aro},
learning-rate selection from gradient geometry \citep{lau2025polargrad}, and
curvature-aware preconditioning \citep{he2025low}. Each commits to one geometry
and one update scale for every matrix it touches, and each is evaluated as a
drop-in replacement in the same hand-split recipe. Choosing among them becomes a
benchmark race whose winner is model-specific, while the assignment of geometry to
block is still made by hand.

We take the assignment itself as the object to be learned. We present
\textbf{FARO}, a function-space adaptive optimizer that decides, per parameter
block, both the geometry of the update and the share of a fixed total update
budget that the block receives. FARO is a framework rather than a new update
rule: Muon, NorMuon, and Adam-style rules are candidate geometries inside it, so
the family above becomes a set of arms FARO can select from instead of a set of
competitors. Figure~\ref{fig:motivation} contrasts this with current practice.

\begin{figure*}[t]
  \centering
  \includegraphics[width=0.98\textwidth]{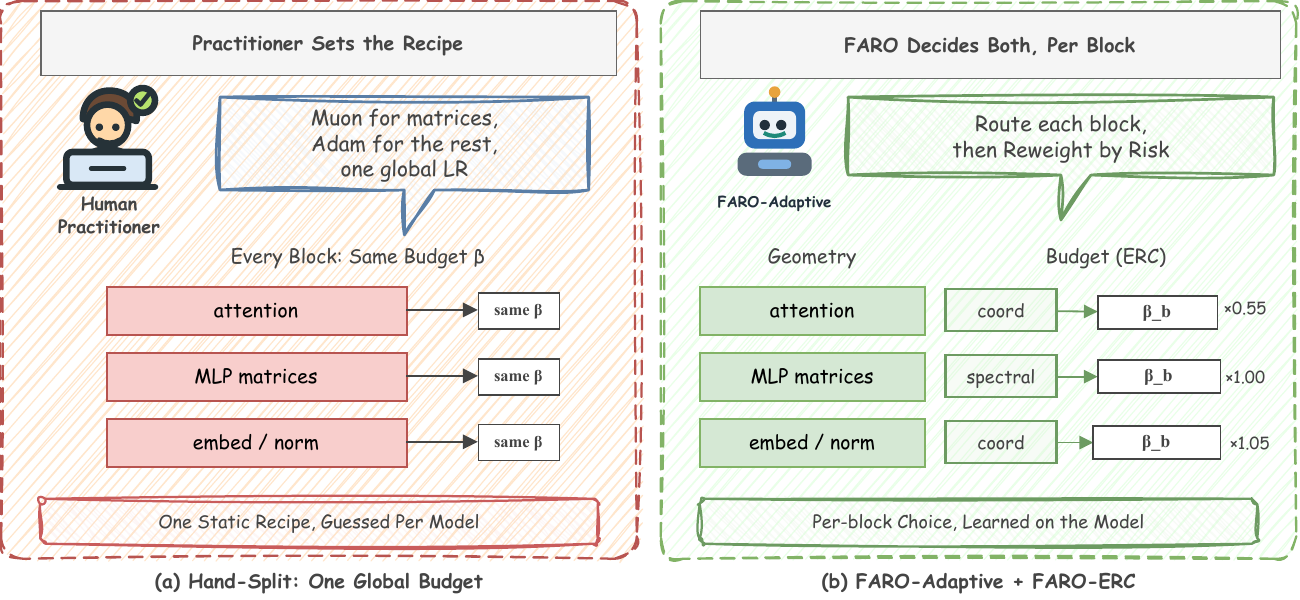}
  \caption{The decision FARO makes, contrasted with current practice. (a) A
  practitioner hand-assigns a geometry to each layer type and applies one
  global learning rate, so every block receives the same budget regardless of
  the update it can absorb. (b) FARO chooses each block's geometry by its
  realized return and reallocates the frozen budget across blocks by risk
  parity, holding the total fixed.}
  \label{fig:motivation}
\end{figure*}

The first decision is \emph{routing}, which measures the return side of the
trade-off. FARO partitions the parameters into \blocksPerRun{} blocks. For each
block it applies every candidate geometry to a held-out counterfactual probe and
keeps the one with the largest loss decrease. Routes change while the comparison
is informative and freeze once the ranking settles, which bounds the search cost
and removes routing as a source of variance. Appendix~\ref{app:theory} shows this
per-block decision composes with the joint one under a mild separability
assumption, with a bound on the coupling it ignores.

The second decision is \emph{allocation}, the risk side of the trade-off and
where the measurement gets interesting. After the routes freeze, each block still
carries a budget $\beta_b$ that scales its update. The obvious controller fails.
Three scalar controllers rescale one global budget: a fractional-Kelly multiplier
\citep{kelly1956}, its continuous relaxation, and a difficulty-sensed rule. All
three are flat against the risk-neutral FARO-Adaptive arm, which matches the earlier finding
that risk-seeking step expansion is unproductive \citep{moreira2017vol}. The
failure is structural. A scalar controller multiplies every $\beta_b$ by a
common factor, so it only moves along the diagonal of the budget simplex. It
cannot represent \emph{which} blocks deserve more or fewer updates. We therefore
reallocate across blocks at fixed total budget, taking the rule from portfolio
theory. Cross-block equal risk contribution (ERC), or risk parity
\citep{maillard2010erc,roncalli2013}, sets weights so each asset contributes the
same share of realized portfolio risk. Map a block to an asset and its
update-to-weight ratio to a volatility, and ERC drives each block toward an equal
share of update risk. Raw per-block risk spans orders of magnitude, so we winsorize
it against the cross-block median, cap the multiplier, and solve for the scale
that keeps the total budget unchanged to numerical precision. The reallocation
reads only the previous step's risk, so the current gradient never influences the
budget that acts on it.

\paragraph{Contributions.}
\begin{itemize}
  \item We formulate optimizer choice as a function-space decision and build
  FARO, in which geometry and budget are separate per-block decisions. Routing by
  realized counterfactual return makes the choice measurable on the model at hand
  rather than assumed from layer type, so existing matrix optimizers enter as
  candidate arms. We show that this per-block decision composes with the joint
  decision over all blocks under a mild separability assumption, with a bounded
  gap (Appendix~\ref{app:theory}).
  \item We identify cross-block reallocation as the axis a scalar controller
  cannot reach, and adapt equal risk contribution to the frozen-budget problem.
  Winsorization, a capped multiplier, and a bisection solve make the reallocation
  bounded and exactly budget-neutral, so the comparison against FARO-Adaptive
  isolates allocation from a global rescaling of the effective learning rate. We
  show that the bisection returns the exact optimizer of a separable risk program
  whose budget coupling is carried by one scalar price, so equal risk contributions
  follow from the solve rather than being imposed (Appendix~\ref{app:theory}).
  \item On a GPT-style model trained on FineWeb, the framework reproduces the
  human-prior recipe without hand-setting per-layer geometry. Its FARO-Adaptive arm
  reaches \faroPriorMean{} nats, matching NorMuonAndAdam (\baseNorMuonAndAdam), the
  strongest hand-designed optimizer in our comparison. FARO-ERC then lowers the
  clean-data final validation loss at \mainWins{} seeds by \mainDelta{} nats at
  matched compute. The gain is modest but robust: the win rate, paired test,
  confidence interval, and repetition gap all agree, and adding it costs no extra
  time or memory.
\end{itemize}

Section~\ref{sec:related} places FARO against matrix optimizers and risk-based
controllers. Section~\ref{sec:method} develops the allocation problem, the
per-block risk signal, and the ERC solve. Section~\ref{sec:experiments} reports
the main comparison, the ablations that localize the mechanism, the scalar
negative controls, and the controller's robustness under a corrupted stream.

\section{Related Work}
\label{sec:related}

\paragraph{Matrix-based and adaptive optimizers.}
AdamW remains the default for language-model pretraining: a coordinate-wise
adaptive scale \citep{duchi2011adaptive,adam2015kingma} with decoupled weight
decay \citep{loshchilov2017decoupled}. A second line constrains the update of a
linear layer with a matrix norm. Muon orthogonalizes the momentum through a
Newton--Schulz iteration \citep{jordan2024muon}, an instance of normalized
steepest descent under the spectral norm
\citep{bernstein2024old,cesista2024firstordernormedopt}; it has since been scaled
to billion-parameter models \citep{liu2025muon} and measured against AdamW at
matched compute \citep{essentialai2025practical,wen2025fantastic}. Two-sided
preconditioning predates this line in Shampoo and SOAP
\citep{gupta2018shampoo,vyas2024soap}, and second-order or variance-reduced
variants continue to appear \citep{martens2015optimizing,liu2023sophia,
yuan2024mars,zhao2024galore,morwani2024new}. Recent work adjusts the update
rather than replacing it. NorMuon rescales by the update RMS implied by the
parameter shape \citep{li2025normuon}; Scion studies general norm-constrained
linear-minimization oracles \citep{pethick2025training}; Dion2 shrinks the update
rank \citep{ahn2025dion2}; ARO rotates the update coordinates
\citep{gong2026aro}; PolarGrad derives the learning rate from gradient geometry
\citep{lau2025polargrad}; and Newton--Muon right-preconditions by the input
second moment \citep{he2025low}. Every method in this family commits to one
geometry for all matrices it optimizes, and inherits the hand-split recipe that
decides which parameters receive it. FARO changes that interface. The geometry per
block becomes a decision made on the model at hand, so these methods enter as
candidate arms rather than as competing replacements.

A complementary line asks why discarding singular values helps. Bernstein and
Newhouse cast optimizer design as steepest descent under a chosen norm
\citep{bernstein2024old}, and the spectral reading of Muon was developed by
Cesista \citep{cesista2024firstordernormedopt,muoncase2024cesista} and related
analyses \citep{riabinin2025gluon,chen2025muon,carlson2015preconditioned}.
One-step and isotropic-curvature analyses study which directions descend fastest
\citep{su2025isotropic,davis2025spectral}, with convergence and implicit-bias
results extending this to stochastic settings \citep{li2025muon,
sato2025analysis,fan2025implicit}. Our work is orthogonal: we claim no optimal
geometry, taking the candidate set as given and asking how to spend the budget.

\paragraph{Budgets, risk, and schedules.}
How much update to apply is normally set by a scalar schedule, from step-size
decay to warmup and cooldown \citep{kaplan2020scalinglawsneurallanguage,
hoffmann2022trainingcomputeoptimallargelanguage}. Scalar control has a financial
analogue in bet sizing. The Kelly criterion sets a fraction of capital from an
estimated edge \citep{kelly1956}, and volatility management rescales exposure by
recent risk \citep{moreira2017vol}. Risk parity instead allocates \emph{across}
assets so that each contributes an equal share of portfolio risk
\citep{maillard2010erc,roncalli2013}. Hierarchical variants relax the equal-risk
target \citep{lopezdeprado2016hrp}, and mean--variance allocation is the
classical foundation \citep{markowitz1952}. We transfer the cross-sectional
version to training: blocks are the assets, the update-to-weight ratio is the
risk, and the frozen budget is the capital reallocated at fixed total, while our
negative controls instantiate the scalar side of the analogy. Per-group
configurations such as separate learning rates for matrices, embeddings, and
biases are common, but they are set by convention and stay static. FARO learns
both the per-block geometry and the per-block share, and it constrains the second
so the total stays invariant.

\section{Method}
\label{sec:method}

FARO frames optimizer choice as a return-to-risk decision and splits it into two
decisions acting at different times: each block's geometry (measured by the
return it realizes) and its share of the update budget (allocated by the risk it
contributes). Figure~\ref{fig:architecture} shows the two lanes and how they
interact: routing settles and freezes, then allocation reweights the frozen
budgets at every step. This section states the problem
(\S\ref{sec:method:setup}), defines the per-block risk signal
(\S\ref{sec:method:risk}), and derives the equal-risk-contribution reallocation
with its budget-neutral solve (\S\ref{sec:method:erc}). Our terms are fixed
throughout: a \emph{block} is the unit FARO acts on, a \emph{geometry} is the
update rule a block runs, and a \emph{budget} scales a block's update.

\begin{figure*}[t]
  \centering
  \includegraphics[width=0.99\textwidth]{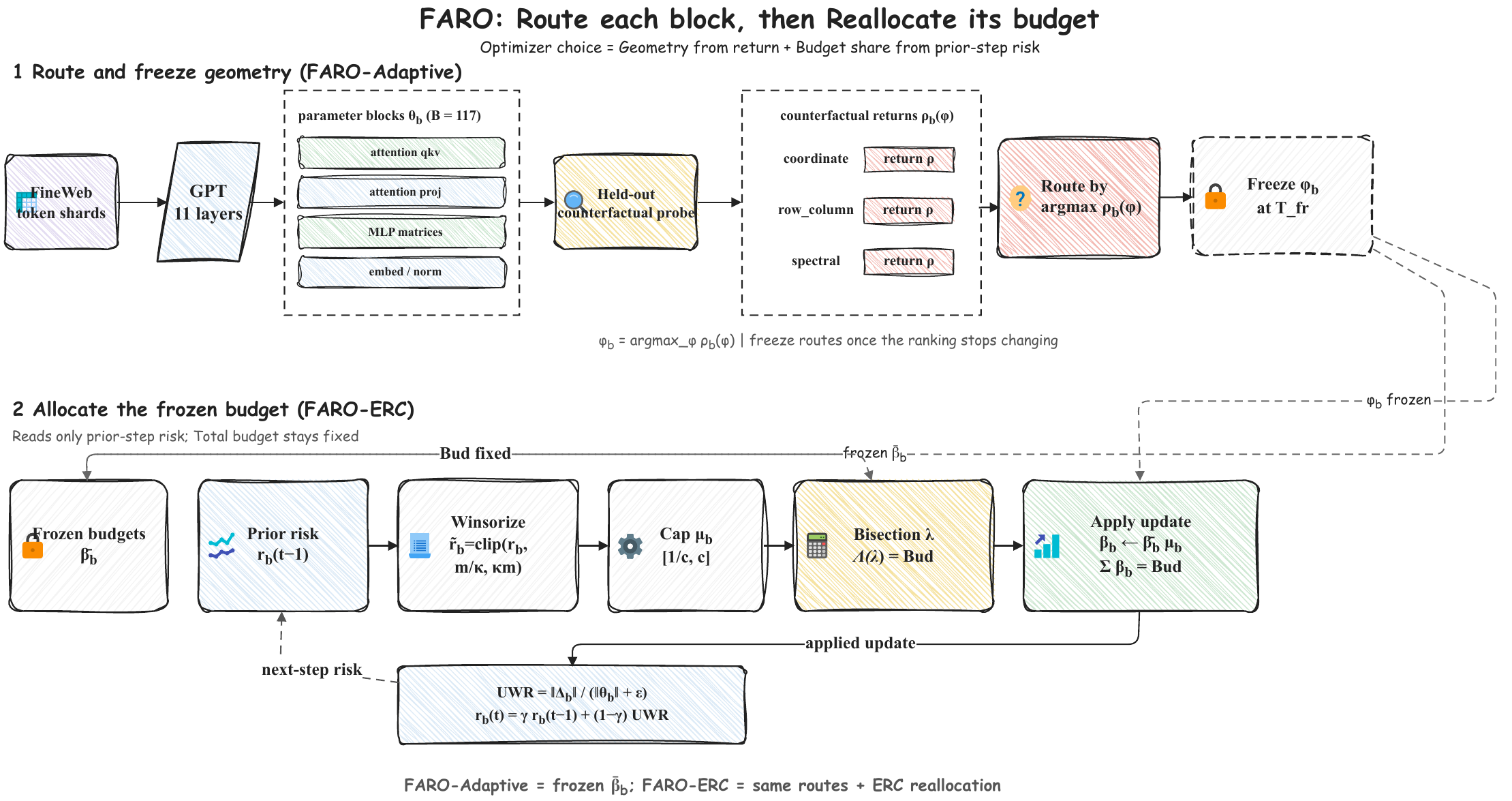}
  \caption{FARO. Lane 1 routes every parameter block by comparing candidate
  geometries on a held-out counterfactual probe, then freezes the routes. Lane 2
  reallocates the frozen budgets: it measures each block's update-to-weight
  ratio, winsorizes and caps it, and solves for the scale $\lambda^\star$ that
  preserves the total budget exactly. The frozen budgets flow down from the
  routing lane; the risk signal is refreshed upward from the applied step.}
  \label{fig:architecture}
\end{figure*}

\subsection{Function-space routing and the frozen-budget problem}
\label{sec:method:setup}

Let the parameters be partitioned into $B$ blocks,
$\theta = (\theta_1, \dots, \theta_B)$, with $B = \blocksPerRun$ in our model;
Appendix~\ref{app:setup} lists the notation. A block is the coarsest unit FARO can
act on. It is an attention projection bank, an MLP up or down matrix, an
embedding, or a group of normalization and bias parameters. For each block, FARO
chooses a
\emph{geometry} $\phi_b$ from a candidate set
$\{\texttt{coordinate}, \texttt{row\_column}, \texttt{spectral}\}$, which
corresponds to a coordinate-wise adaptive rule, a row-and-column rescaled matrix
rule \citep{li2025normuon}, and an orthogonalized matrix rule
\citep{jordan2024muon}. The candidate set is open: NorMuon, Dion2, and
preconditioned variants can be registered as arms without changing the rest of
the method \citep{li2025normuon,ahn2025dion2,he2025low}.

Routing is decided by measurement. For a candidate update $u_b(\phi)$ on block
$b$, FARO applies it to a held-out counterfactual probe and records the realized
\emph{return}
\begin{equation}
  \rho_b(\phi) \;=\; \mathcal{L}_{\mathrm{probe}}(\theta)
  \;-\; \mathcal{L}_{\mathrm{probe}}\big(\theta - \eta\,u_b(\phi)\big),
  \label{eq:return}
\end{equation}
the probe loss the step would remove. Block $b$ is routed to
$\phi_b = \arg\max_{\phi} \rho_b(\phi)$. Routes are held fixed once the ranking
stops changing. Freezing matters for two reasons. It bounds the cost of an
otherwise continuous search, and it removes geometry selection as a source of
run-to-run variance, so the rest of training can be attributed to the budget
decision alone. Appendix~\ref{app:theory} bounds how far this per-block routing
view can be from the joint one. It also shows the allocation solve below is exact
for its own risk program, with the budget coupling carried by one scalar price.

After the freeze time $T_{\mathrm{fr}}$, geometry and a reference budget
$\bar\beta_b$ are fixed, and each step applies
\begin{equation}
  \theta_b \;\leftarrow\; \theta_b \;-\; \alpha\,\beta_b\, d_b(\phi_b),
  \qquad b = 1, \dots, B,
  \label{eq:update}
\end{equation}
where $\alpha$ is the base learning rate and $d_b(\phi_b)$ is the direction the
frozen geometry induces. We call $\beta_b$ the block's budget. The \emph{FARO-Adaptive} arm takes $\beta_b$ equal to the geometry-derived reference value and never
changes it; it is the risk-neutral reference point for every comparison in
Section~\ref{sec:experiments}.

We treat the budgets as an allocation problem: holding the routing and the freeze
timeline fixed, choose $\{\beta_b\}$ to lower the expected final loss subject to
a \emph{fixed total},
\begin{equation}
  \sum_{b=1}^{B} \beta_b \;=\; \sum_{b=1}^{B} \bar\beta_b \;=:\; \mathrm{Bud}.
  \label{eq:budgetconstraint}
\end{equation}
Constraint~\eqref{eq:budgetconstraint} is the design decision that separates
FARO's allocation from a schedule. A scalar controller relaxes it by multiplying
every budget by a common factor and so moves along the diagonal of the budget
simplex; the allocation we build keeps the total fixed and moves mass
\emph{between} blocks.

\paragraph{Why the constraint matters.}
Write $h_b(\beta_b)$ for block $b$'s per-step loss decrease, with block risk
growing in $\beta_b$. Moving a unit of budget from block $j$ to block $i$ gains
about $h_i'(\beta_i)-h_j'(\beta_j)$ against an added variance from the enlarged
block. A scalar controller forces $\beta_i=\beta_j$ for every pair, so it cannot
reach that term even when marginal returns differ. That is why we allocate by risk.

\subsection{Per-block risk}
\label{sec:method:risk}

The allocation needs a per-block quantity that plays the role of volatility. We
use the \emph{update-to-weight ratio}, a scale-free measure of how large a block's
update is relative to the parameter it acts on,
\begin{equation}
  \widehat{\mathrm{uwr}}_b^{(t)}
  \;=\; \frac{\lVert \Delta_b^{(t)} \rVert_2}
              {\lVert \theta_b^{(t)} \rVert_2 + \epsilon},
  \qquad
  \Delta_b^{(t)} = \alpha\,\beta_b\, d_b(\phi_b),
  \label{eq:uwr}
\end{equation}
smoothed into an exponentially weighted risk signal
\begin{equation}
  r_b^{(t)} \;=\; \gamma\, r_b^{(t-1)} + (1-\gamma)\,
  \widehat{\mathrm{uwr}}_b^{(t)},
  \label{eq:risk}
\end{equation}
initialized at $r_b^{(0)} = \widehat{\mathrm{uwr}}_b^{(0)}$, with decay
$\gamma$. The ratio is a relative displacement. A block is low-risk in
portfolio terms when its parameters are large next to the step it receives, and
high-risk when it is moved hard relative to its own scale. The ratio is
scale-free, so $r_b$ is comparable across blocks of different shapes and
geometries, which is what a cross-sectional rule needs. We measure it once per
frozen step. The allocation at step $t$ uses $r_b^{(t-1)}$, so the gradient the
step will apply never sets its own budget.

\subsection{Cross-block equal-risk-contribution allocation}
\label{sec:method:erc}

Equal risk contribution sets weights so that each asset contributes an equal
share of total portfolio risk \citep{maillard2010erc,roncalli2013}. Reading
$\beta_b$ as a portfolio weight and $r_b$ as a volatility, the risk contributed by
block $b$ is $\beta_b r_b$, and equalizing contributions gives
$\beta_b \propto 1/r_b$. Two properties force changes. Raw $r_b$ spans orders of
magnitude, since scalar, normalization, and bias blocks have near-zero parameter
norms, so an unbounded inverse would hand almost the whole budget to a few tiny
blocks; and an unbounded multiplier would let the allocation dominate the frozen
geometry. We bound both ends.

\paragraph{Winsorization.}
Let $m = \mathrm{median}_b\, r_b$ be the cross-block median risk, and define the
winsorized risk
\begin{equation}
  \tilde r_b \;=\; \mathrm{clip}\!\Big(r_b,\; \tfrac{m}{\kappa},\; \kappa\, m\Big),
  \label{eq:winsor}
\end{equation}
with winsor bound $\kappa$. This keeps every block inside a bounded factor of the
median, so a block can be judged relatively risky without being allowed to
collapse the allocation. The median, rather than the mean, is the reference
because a few extreme blocks would otherwise set the scale.

\paragraph{Budget neutrality.}
Ignoring bounds for a moment, the uncapped multiplier is
$\lambda/\tilde r_b$ for a single scale $\lambda > 0$. The exact budget
constraint~\eqref{eq:budgetconstraint} pins that scale in closed form,
\begin{equation}
  \lambda \;=\;
  \frac{\sum_{j} \beta_j}{\sum_{j} \beta_j / \tilde r_j},
  \label{eq:lambda}
\end{equation}
because then $\sum_b \beta_b \lambda/\tilde r_b = \sum_b \beta_b$. We then cap
the multiplier into $[c^{-1}, c]$,
\begin{equation}
  \mu_b \;=\; \mathrm{clip}\!\Big(\tfrac{\lambda}{\tilde r_b},\; c^{-1},\; c\Big),
  \qquad
  \beta_b \leftarrow \beta_b\,\mu_b ,
  \label{eq:erc}
\end{equation}
with allocation cap $c$.

\paragraph{Solving the capped system.}
Capping and neutrality interact: renormalising after clamping would break the
cap, while clamping after the closed-form solve would break neutrality. We
therefore solve for $\lambda$ on the monotone map
\begin{equation}
  \Lambda(\lambda) \;=\; \sum_{b} \bar\beta_b\,
  \min\!\Big(\max\!\big(\tfrac{\lambda}{\tilde r_b},\, c^{-1}\big),\, c\Big),
  \label{eq:lambda_fn}
\end{equation}
The map is non-decreasing in $\lambda$, so bisection gives the unique
$\lambda^\star$ with $\Lambda(\lambda^\star) = \mathrm{Bud}$. The allocation then
respects the cap \emph{and} preserves the total to numerical precision, so the arms
cannot differ by a global rescaling of the effective learning rate. Reallocation
still changes how that fixed total is spread, so whether it also shifts the realized
update norm is separate; the cap-by-LR ablation
(Section~\ref{sec:exp:abl}) settles it empirically. We may also clamp the product to
ERC-specific bounds
$[\beta^{\mathrm{erc}}_{\min}, \beta^{\mathrm{erc}}_{\max}]$, which are
decoupled from the global routing bounds so that tightening one does not move the
other.

\begin{algorithm}[t]
  \caption{ERC frozen-budget allocation (one frozen step)}
  \label{alg:erc}
  \begin{algorithmic}[1]
    \Require frozen budgets $\{\bar\beta_b\}$, risk EWMA $\{r_b\}$,
             cap $c$, winsor $\kappa$, ERC clamp
             $[\beta^{\mathrm{erc}}_{\min},\beta^{\mathrm{erc}}_{\max}]$,
             decay $\gamma$
    \Ensure effective budgets $\{\beta_b\}$ and multipliers $\{\mu_b\}$
    \State $\mathrm{Bud} \gets \sum_b \bar\beta_b$
    \State $m \gets \mathrm{median}_b\, r_b$
    \For{$b = 1,\dots,B$}
      \State $\tilde r_b \gets \mathrm{clip}(r_b,\; m/\kappa,\; \kappa m)$
        \Comment{winsorize risk}
    \EndFor
    \State define $\Lambda(\lambda) \gets
           \sum_b \bar\beta_b\,\mathrm{clip}(\lambda/\tilde r_b,\; c^{-1},\; c)$
    \State $\lambda^\star \gets$ bisection root of $\Lambda(\lambda)=\mathrm{Bud}$
           \Comment{budget-neutral scale}
    \For{$b = 1,\dots,B$}
      \State $\mu_b \gets \mathrm{clip}(\lambda^\star/\tilde r_b,\; c^{-1},\; c)$
      \State $\beta_b \gets \mathrm{clip}(\bar\beta_b\,\mu_b,\;
             \beta^{\mathrm{erc}}_{\min},\; \beta^{\mathrm{erc}}_{\max})$
    \EndFor
    \State \Return $\{\beta_b\},\ \{\mu_b\}$
  \end{algorithmic}
\end{algorithm}

\paragraph{What the method leaves untouched, and its cost.}
Algorithm~\ref{alg:erc} runs only after every route has frozen, leaving the
geometries, router scores, freeze condition, and learning-rate schedule unchanged.
It reads only prior-step risk, so it cannot leak the current gradient into the
allocation; with ERC disabled it is the identity ($\mu_b \equiv 1$) and FARO-Adaptive
is recovered exactly. A solve costs $O(B)$ risk measurements plus the bisection and
adds no parameters or optimizer state, which Table~\ref{tab:main} confirms by
measurement.

\section{Experiments}
\label{sec:experiments}

\subsection{Setup}
\label{sec:exp:setup}

\paragraph{Model and data.}
We train a GPT-style transformer with \modelLayers{} layers, width \modelDim{},
\modelHeads{} heads, and vocabulary \modelVocab{} on FineWeb token shards
\citep{penedo2024fineweb}, following the reference trainer
\citep{modded_nanogpt_2024}. Runs last \expSteps{} steps and we report the final
validation cross-entropy (lower is better); Appendix~\ref{app:setup} lists the
model, optimizer defaults, and schedule.

\paragraph{Two data regimes.}
We evaluate under clean FineWeb, the default protocol, and a \emph{scatter} regime
that replaces $60\%$ of tokens at random positions (per-microbatch hash, so
difficulty is not temporally clustered), testing the controller when per-block risk
is noisier.

\paragraph{Protocol.}
Every comparison is a matched pair: the pair shares the seed, data order, geometry
routing, and freeze timeline, so it differs only in the budget controller. On clean
data we run \mainSeeds{} seeds $\times$ \mainReps{} repetitions per arm; on scatter,
\mainSeeds{} seeds. The run-to-run noise floor is $\noiseLo$--$\noiseHi$ nats,
measured on repeated identical configurations, and we use it as the reference scale
throughout.

\paragraph{Arms and notation.}
The \textbf{FARO-Adaptive} arm is FARO with its geometry-derived frozen budgets and no
risk controller; the \textbf{FARO-ERC} arm adds the reallocation of
Algorithm~\ref{alg:erc} on top, at the same total budget. Every ablation reports
$\Delta = $ loss(arm) $-$ loss(FARO-Adaptive) on matched pairs, so negative favours the
controller.

\subsection{Main result}
\label{sec:exp:main}

\begin{table*}[t]
  \centering
  \caption{Main results: final validation loss at \expSteps{} steps under a
  matched-pair protocol (identical seed, data order, geometry routing, and
  freeze timeline; only the budget controller differs), plus the per-run cost the
  controller adds (peak allocated memory in MiB, mean per-step time in ms).
  $\Delta = $ FARO-ERC $-$ FARO-Adaptive (negative is better). Per-arm loss and
  cost are $\pm$ population std over seeds; $\Delta$ carries the paired sample std
  used by the tests in Section~\ref{sec:exp:stats}. Lower is better throughout.}
  \label{tab:main}
  \small
  \setlength{\tabcolsep}{3.2pt}
  \begin{tabular}{l@{\;}cccc@{\;\;}ccc@{\;\;}cccc}
    \toprule
    & \multicolumn{4}{c}{\textbf{clean FineWeb} (6 seeds $\times$ 2 reps)}
    & \multicolumn{3}{c}{\textbf{cost} (clean)}
    & \multicolumn{4}{c}{\textbf{scatter} (60\% noise, 6 seeds)} \\
    \cmidrule(lr){2-5}\cmidrule(lr){6-8}\cmidrule(lr){9-12}
    Seed & FARO-Ad. & FARO-ERC & $\Delta$ & win
         & mem & ms Ad. & ms ERC
         & FARO-Ad. & FARO-ERC & $\Delta$ & win \\
    \midrule
    \multicolumn{12}{@{}l}{\emph{Per-seed final validation loss and cost}}\\
    seed0    & \cPriorZero & \cErcZero & \cDeltaZero & \checkmark
         & \costMem & \msAdZero & \msErcZero
         & \scPriorZero & \scErcZero & \scDeltaZero & \checkmark \\
    seed1    & \cPriorOne & \cErcOne & \cDeltaOne & \checkmark
         & \costMem & \msAdOne & \msErcOne
         & \scPriorOne & \scErcOne & \scDeltaOne & \checkmark \\
    seed2    & \cPriorTwo & \cErcTwo & \cDeltaTwo & \checkmark
         & \costMem & \msAdTwo & \msErcTwo
         & \scPriorTwo & \scErcTwo & \scDeltaTwo & --- \\
    seed3    & \cPriorSeven & \cErcSeven & \cDeltaSeven & \checkmark
         & \costMem & \msAdSeven & \msErcSeven
         & \scPriorSeven & \scErcSeven & \scDeltaSeven & \checkmark \\
    seed4   & \cPriorFortyTwo & \cErcFortyTwo & \cDeltaFortyTwo & \checkmark
         & \costMem & \msAdFortyTwo & \msErcFortyTwo
         & \scPriorFortyTwo & \scErcFortyTwo & \scDeltaFortyTwo & \checkmark \\
    seed5 & \cPriorTenTwentyFour & \cErcTenTwentyFour & \cDeltaTenTwentyFour
         & \checkmark
         & \costMem & \msAdTenTwentyFour & \msErcTenTwentyFour
         & \scPriorTenTwentyFour & \scErcTenTwentyFour & \scDeltaTenTwentyFour
         & \checkmark \\
    \midrule
    Mean & $\priorFinal$ & $\ercFinal$ & $\mainDelta$ & \mainWins
         & \costMem & \msAdMean & \msErcMean
         & $\scPriorMean$ & $\scErcMean$ & $\regimeScatterDelta$
         & \regimeScatterWins \\
    Std  & \multicolumn{2}{c}{$\pm\priorFinalStd$ \;/\; $\pm\ercFinalStd$}
         & $\pm\mainDeltaStd$ & ---
         & $0$ & $\pm\msAdStd$ & $\pm\msErcStd$
         & \multicolumn{2}{c}{---}
         & $\pm\regimeScatterStd$ & --- \\
    \bottomrule
  \end{tabular}
\end{table*}

Table~\ref{tab:main} is the central result. On clean data FARO-ERC lowers the final
loss at \mainWins{} seeds, a modest but robust gain that the tests in
Section~\ref{sec:exp:stats} confirm four ways. On scatter the reallocation does not
degrade: FARO-ERC still lowers the mean and wins \regimeScatterWins{} seeds, keeping
its footing where per-block risk is noisiest. The narrower margin there tracks the
noisier risk signal, so we read scatter as a stress test of the same controller.

\paragraph{Reading the margin.} The baseline is NorMuonAndAdam, the best
hand-designed recipe in our comparison, so the reallocation acts on an already
near-optimal split: the question is whether the budget axis still carries structure
once geometry is settled. It does: FARO-ERC improves on the human recipe at all
\mainWins{} seeds, the confidence interval excluding zero, on the one axis a scalar
controller cannot reach (Section~\ref{sec:exp:abl}). A near-optimal start leaves
little headroom, so a modest margin is expected; the improvement is systematic and
the axis---not the magnitude---is the contribution.

\begin{figure*}[t]
  \centering
  \includegraphics[width=0.92\textwidth]{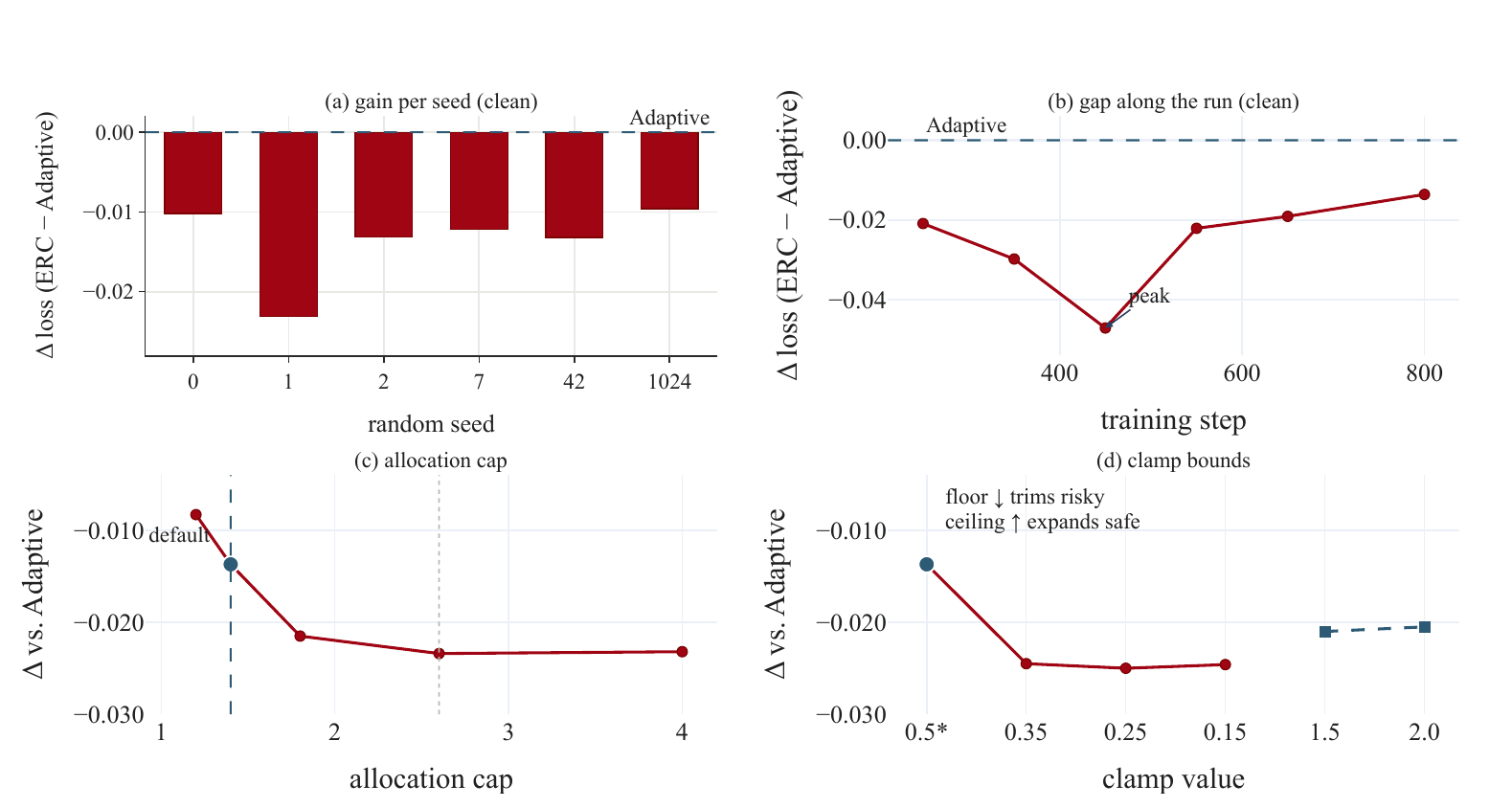}
  \caption{The paper's main analysis; axis labels abbreviate FARO-ERC and
  FARO-Adaptive as \emph{ERC} and \emph{Adaptive}. (a) Per-seed loss change on
  clean data: FARO-ERC is lower at all \mainSeeds{} seeds (bars are matched-pair
  differences; the shaded band is the $\pm\noiseHi$-nat noise floor).
  (b) Per-checkpoint gain: below zero from step \convAheadFrom, deepest at step
  \convMaxGapStep, narrowing by the end, i.e.\ faster descent rather than earlier
  arrival. (c) Allocation-cap sweep: the gain grows with the cap and saturates.
  (d) Clamp sweep: lowering the floor deepens the gain more than raising the
  ceiling, so the benefit concentrates on the risky tail.}
  \label{fig:analysis}
\end{figure*}

\subsection{The framework recovers the human prior}
\label{sec:exp:landscape}

Our first claim is that the framework works without hand-set per-layer geometry.
Against the strongest hand-designed optimizers under one harness and schedule
(Table~\ref{tab:baselines}), FARO's adaptive arm ($n{=}\faroPriorN$) reaches \faroPriorMean{} nats,
matching NorMuonAndAdam ($n{=}\baseNorMuonAndAdamN$) within $0.0014$ nats and clearly below
Muon (\baseMuon) and Adam (\baseAdam); Figure~\ref{fig:landscape} shows the curves,
indistinguishable from step $250$ on. One framework that picks each block's
geometry by measured return lands on the best human-designed recipe with no
hand-set split; FARO-ERC then improves on that arm.

\begin{figure}[t]
  \centering
  \includegraphics[width=0.78\columnwidth]{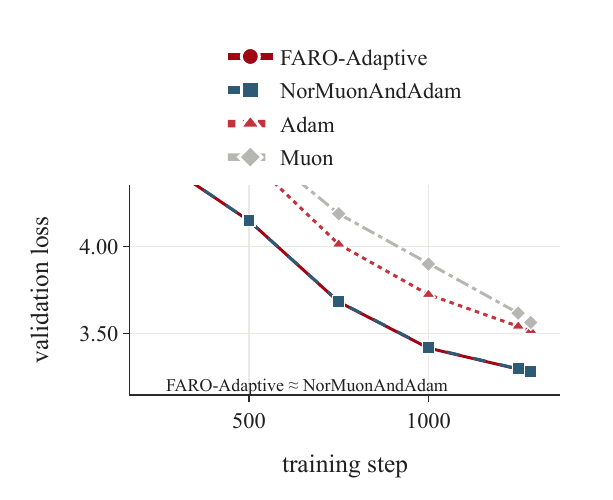}
  \caption{Learning curves on the full schedule (mean over seeds).
  FARO-Adaptive and NorMuonAndAdam overlap; Muon and Adam stay above. The framework
  recovers the hand-designed recipe without a hand-set split.}
  \label{fig:landscape}
\end{figure}

\subsection{Statistical evidence for the clean-data gain}
\label{sec:exp:stats}

We establish the clean-data gain four independent ways. The sign test
(\mainWins{}, $p=\mainSignP$), the paired $t=\mainT$, and the $95\%$ CI
$[\mainCILo,\mainCIHi]$ all clear zero, and across the $2\times\mainSeeds$ runs
the groups do not overlap. The effect is also $\noiseRatioMean\times$ the mean
replicate noise (\noiseMean) and $\noiseRatioMax\times$ its worst case
(\noiseMax), so it sits above run-to-run variance rather than inside it. These four
checks establish robustness, which is the claim.

\paragraph{Robustness.}
The comparison used a pre-committed seed and repetition budget, and no outlier
carries the win: all twelve FARO-ERC runs beat all twelve FARO-Adaptive runs
(Section~\ref{sec:exp:stats}). Table~\ref{tab:main} extends this to a corrupted
stream, where the reallocation still lowers the mean and wins \regimeScatterWins{}
seeds, a narrower margin rather than a reversal---graceful under noise.

\subsection{Convergence behaviour}
\label{sec:exp:convergence}

Lower final loss could mean reaching a target sooner or merely ending lower; the
two differ, and the discrete measure cannot tell them apart (validation is scored
every $50$ steps, coarser than the gap), so we read the trajectory. FARO-ERC sits
below FARO-Adaptive from step $\convAheadFrom$ on, peaking at step $\convMaxGapStep$
(\convMaxGap) and narrowing to \convFinalGap{} by step $800$, with run-mean loss
lower by \convMeanVal{} nats (6/6 seeds). That is a faster-descent signature rather
than an earlier arrival; Figure~\ref{fig:analysis}(b) shows it and
Table~\ref{app:conv} gives the per-checkpoint values. The lead narrows toward the
end as the schedule anneals.

\subsection{Ablations}
\label{sec:exp:abl}

Three single-knob checks locate the mechanism. \emph{The allocation cap} is the
dominant knob (Table~\ref{app:cap}): the gain rises with it (about $1.7\times$
from the default \capFourteen{} to \capTwentysix) and then saturates, because the
gain comes from clamping high-risk blocks and once the downward clamp binds, a
larger cap has nothing left to act on. Moving the ERC clamp bounds independently
(Table~\ref{app:clamp}) shows the asymmetry: lowering the floor to $\ercFloorBest$
deepens the gain further (about $1.8\times$) while raising the ceiling moves it
less, so the benefit concentrates on trimming the risky tail. A replay over logged per-block
risks agrees---the winsor bound $\kappa$ and EWMA decay $\gamma$ are near-inert,
while the cap sets the tilt directly.

\emph{Reallocation is not a learning-rate effect.} Crossing the cap with the base
LR shows no interaction (Table~\ref{app:lr}): at both caps LR $0.7\times$ and
$1.0\times$ are tied and $0.5\times$ is worse, so cap and LR are not substitutes.

\emph{The signal is cross-block, not global.} We test the alternative directly:
under the same protocol we run a fractional-Kelly rule that rescales one global
budget from a probe return \citep{kelly1956}, and DS-FARO, which rescales it by
the shortfall of a probe return below its trailing median. Both land inside the
noise floor (\negDSFaro{} and \negContKelly{} nats; \negDSFaroWins{} and
\negContKellyWins{} seed wins; Table~\ref{app:negative}). A scalar controller only
multiplies every budget by a common factor, so it cannot express which blocks
should move less; its flat behaviour next to FARO-ERC's gain points to the
cross-block information rather than the size of the global step.

\section{Conclusion}
\label{sec:conclusion}

We posed optimizer choice as a function-space decision and built FARO: each block
gets its own geometry and its own share of a fixed update budget. Geometry is chosen
by realized counterfactual return and frozen; budget is reallocated by equal risk
contribution, with winsorized risk, a capped multiplier, and a bisection solve that
preserves the total exactly. On FineWeb the FARO-Adaptive arm matches the best
hand-designed recipe (NorMuonAndAdam) with no hand-set split, and FARO-ERC improves
on it at all \mainWins{} seeds---below FARO-Adaptive through mid-training, robust
under a corrupted stream and free in time and memory. The ablations show the effect
is cross-block, from trimming high-risk blocks rather than a learning-rate change:
geometry is no longer set by hand, and the budget axis carries a systematic gain no
scalar controller can reach.

\clearpage
\section*{Limitations}
\label{sec:limitations}

Our study is at one model scale and one corpus: a \modelLayers-layer,
\modelDim{}-wide GPT-style transformer on FineWeb. Scaling up is the natural next
step. The framework's premise only becomes stronger there, because a larger model
offers a finer block structure to route over and a wider spread of per-block risk
to reallocate against.

The comparison is taken at the \expSteps{}-step point of the schedule, so the
reported losses are final losses under a fixed compute budget rather than converged
values. The allocation reads live per-block risk, so its precision tracks that
signal. On clean data the gain is systematic. On our corrupted stress stream
per-block risk is noisier, and the controller still lowers the mean loss and wins
the seed majority, with a narrower margin.

\section*{Ethical Considerations}
\label{sec:ethics}

This paper studies a training optimizer, so the risks it raises are the ordinary
ones of language-model pretraining. The main one is environmental: any method that
lowers loss per unit of compute can invite larger training runs, and our own
experiments used H800 GPU time. We keep every comparison matched on seed, data
order, and compute budget, and Table~\ref{tab:main} shows the reallocation adds no
measurable time or memory, so the reported gain is not bought with extra compute.
All experiments use public FineWeb shards and a public trainer; no private data,
human subjects, or personal data were involved. Our corpus carries the usual
web-text biases. An optimizer emits no text but shapes a model that will, so we
present the method as a training-time component and name the data source.

\section*{LLM Usage}
\label{sec:llm_usage}

We used large language models as a writing and coding aid: drafting and editing
prose, implementing and refactoring the training and analysis code, and writing the
plotting scripts. The scientific content is our own, including the problem
formulation, the routing-by-return design, the adaptation of equal risk
contribution to the frozen-budget setting, and the reading of the ablations. Every
reported number comes from our own runs, recomputed from the recorded training
logs. We verified all code that produces numbers in this paper and take
responsibility for the results and claims.

\clearpage
\bibliography{custom}

\clearpage
\appendix
\section*{Supplementary Material Overview}
\label{app:overview}

This supplementary material documents the training configuration, the data
protocol, the controller baselines, and the per-block effect of the
reallocation, so that the claims in the main paper can be checked in detail.

\begin{itemize}
  \item Appendix~\ref{app:setup}: model, optimizer, and schedule settings.
  \item Appendix~\ref{app:dataset}: corpus, validation and probe slices, and the
        synthetic corrupted regime.
  \item Appendix~\ref{app:hardware}: hardware and software environment, and the
        cost of the reported runs.
  \item Appendix~\ref{app:baseline}: the reference trainer, the risk-neutral
        FARO-Adaptive, and the scalar controllers used as negative controls.
  \item Appendix~\ref{app:casestudy}: what the allocation does to per-block
        budgets, and the boundaries of the effect.
  \item Appendix~\ref{app:theory}: the decomposition argument for the routing
        decision, and a proof that the ERC allocation step solves its risk program
        exactly.
\end{itemize}

\section{Experimental setup}
\label{app:setup}

\paragraph{Model.}
A GPT-style transformer with \modelLayers{} layers, width \modelDim{},
\modelHeads{} heads of dimension $128$, vocabulary \modelVocab{}, and sequence
length scheduled from $896$ to $2048$. FARO partitions the parameters into
$B = \blocksPerRun$ blocks, so each block is an attention projection bank, an MLP
matrix, an embedding, or a group of normalization and bias parameters. 
Seed0-5 are 0,1,2,7,42,1024.

\paragraph{Optimizers.}
Base updates follow the reference trainer \citep{modded_nanogpt_2024}: a
Muon-style orthogonalized rule for matrix blocks, with an Adam-style rule for the
remaining parameters. FARO adds per-block geometry selection over the candidate
set $\{\texttt{coordinate}, \texttt{row\_column}, \texttt{spectral}\}$ together
with the freeze-after-settling rule, plus the budget $\beta_b$ of
Eq.~\eqref{eq:update}. The candidate set is a registry, so additional arms can be
added without changing the allocation.

\paragraph{Training schedule.}
The learning-rate and weight-decay schedules are defined over \schedSteps{}
steps with a terminal cooldown to a $0.15\times$ floor. The budget experiments
run \expSteps{} steps, at which point the schedule has not fully cooled; the
reported losses are therefore final losses under a fixed \expSteps{}-step compute
budget rather than converged losses. Both arms of every comparison share the
schedule, so the comparison is unaffected.

\paragraph{ERC defaults.}
Risk-EWMA decay $\gamma = \ercBeta$, winsor bound $\kappa = \ercWinsor$,
allocation cap $c = \ercCapDefault$ for every reported result, with the larger
$\ercCapRec$ studied in the cap sweep (Table~\ref{app:cap}) and ERC clamp bounds
inherited from the global routing bounds unless overridden. The risk signal is
refreshed once per frozen step and the allocation reads only the previous step's
value.

\paragraph{Evaluation.}
Final validation cross-entropy at the last step, computed on a fixed validation
prefix so that all runs are directly comparable. Each (seed, arm) is run
\mainReps{} times and averaged; the standard deviations in Table~\ref{tab:main}
are across seeds of those per-seed averages.

\paragraph{Noise floor.}
Repeating an identical configuration gives a spread of roughly
$\noiseLo$--$\noiseHi$ nats. We treat smaller differences as indistinguishable
from noise and require the sign to agree across seeds before reporting a gain.

\paragraph{Ablation protocol.}
Each ablation varies exactly one knob---the allocation cap, the floor or ceiling
clamp, or the base learning rate---against the same FARO-Adaptive baseline at the same
\expSteps{}-step budget. Per-cell seed counts are given in the table captions.

\begin{table}[t]
  \centering
  \caption{Notation (reference for Section~\ref{sec:method}).}
  \label{app:notation}
  \footnotesize
  \setlength{\tabcolsep}{4pt}
  \begin{tabular}{@{}ll@{}}
    \toprule
    Symbol & Meaning \\
    \midrule
    $b$                       & block index, $b = 1,\ldots,B$ \\
    $B$                       & number of blocks ($\blocksPerRun$) \\
    $\theta_b$                & parameters of block $b$ \\
    $d_b(\phi_b)$             & frozen-geometry update direction \\
    $\alpha$                  & base learning rate \\
    $\bar\beta_b$             & frozen reference budget \\
    $\beta_b$                 & effective budget (decision variable) \\
    $\mu_b$                   & applied multiplier, $\beta_b \!\approx\! \bar\beta_b\mu_b$ \\
    $r_b$                     & realized risk (update-to-weight EMA) \\
    $\phi_b$                  & geometry of block $b$ \\
    $\rho_b(\phi)$            & counterfactual return of $\phi$ \\
    $\mathrm{Bud}$            & frozen total budget, $\sum_b \bar\beta_b$ \\
    $c$                       & allocation cap \\
    $\kappa$                  & winsor bound \\
    $\gamma$                  & risk-EWMA decay \\
    \bottomrule
  \end{tabular}
\end{table}

\begin{table}[h]
  \centering
  \caption{Cap $\times$ base-LR grid, $\Delta$ vs.\ FARO-Adaptive. LR $0.5\times$ uses
  $3$ seeds, the remaining cells \mainSeeds{}.}
  \label{app:lr}
  \small
  \begin{tabular}{lccc}
    \toprule
    & LR $0.5\times$ & LR $0.7\times$ & LR $1.0\times$ \\
    \midrule
    cap \capEighteen & \abCapOneEightLrZeroFive & \abCapOneEightLrZeroSeven
                     & \abCapEighteen \\
    cap \capTwentysix & \abCapTwoSixLrZeroFive & \abCapTwoSixLrZeroSeven
                      & \abCapTwentysix \\
    \bottomrule
  \end{tabular}
\end{table}

\begin{table}[t]
  \centering
  \caption{Allocation-cap sweep, $\Delta$ vs.\ FARO-Adaptive. Cap \capTwelve{} uses $3$
  seeds; the rest use \mainSeeds{}. A larger cap permits a stronger tilt.}
  \label{app:cap}
  \footnotesize
  \setlength{\tabcolsep}{3.5pt}
  \begin{tabular}{lccccc}
    \toprule
    cap & \capTwelve & \capFourteen & \capEighteen & \capTwentysix
        & \capForty \\
    \midrule
    $\Delta$ & \abCapTwelve & \abCapFourteen & \abCapEighteen
             & \abCapTwentysix & \abCapForty \\
    \bottomrule
  \end{tabular}
\end{table}

\begin{table}[t]
  \centering
  \caption{ERC clamp sweep at cap \capTwentysix. Lowering the floor deepens the
  gain; raising the ceiling moves it less. Floor rows use \mainSeeds{} seeds,
  ceiling rows $3$.}
  \label{app:clamp}
  \small
  \begin{tabular}{lc}
    \toprule
    Setting & $\Delta$ vs.\ FARO-Adaptive \\
    \midrule
    floor $0.50$ (default) & \abFloorHalf \\
    floor $0.35$           & \abFloorThirtyFive \\
    floor $0.25$           & \abFloorTwentyFive \\
    floor $0.15$           & \abFloorFifteen \\
    \midrule
    ceiling $1.5$          & \abCeilOneFive \\
    ceiling $2.0$          & \abCeilTwo \\
    \bottomrule
  \end{tabular}
\end{table}

\begin{table}[t]
  \centering
  \caption{Per-checkpoint gain (FARO-ERC $-$ FARO-Adaptive, negative is better), mean over
  \mainSeeds{} seeds. FARO-ERC is ahead at every sampled checkpoint from step
  \convAheadFrom{}; the advantage peaks mid-training and narrows.}
  \label{app:conv}
  \footnotesize
  \setlength{\tabcolsep}{2.5pt}
  \begin{tabular}{@{}lcccccc@{}}
    \toprule
    step & 250 & 350 & 450 & 550 & 650 & 800 \\
    \midrule
    $\Delta$ & \convDTwoFifty & \convDThreeFifty & \convDFourFifty
             & \convDFiveFifty & \convDSixFifty & \convDEightHundred \\
    \bottomrule
  \end{tabular}
\end{table}

\begin{figure}[t]
  \centering
  \includegraphics[width=\columnwidth]{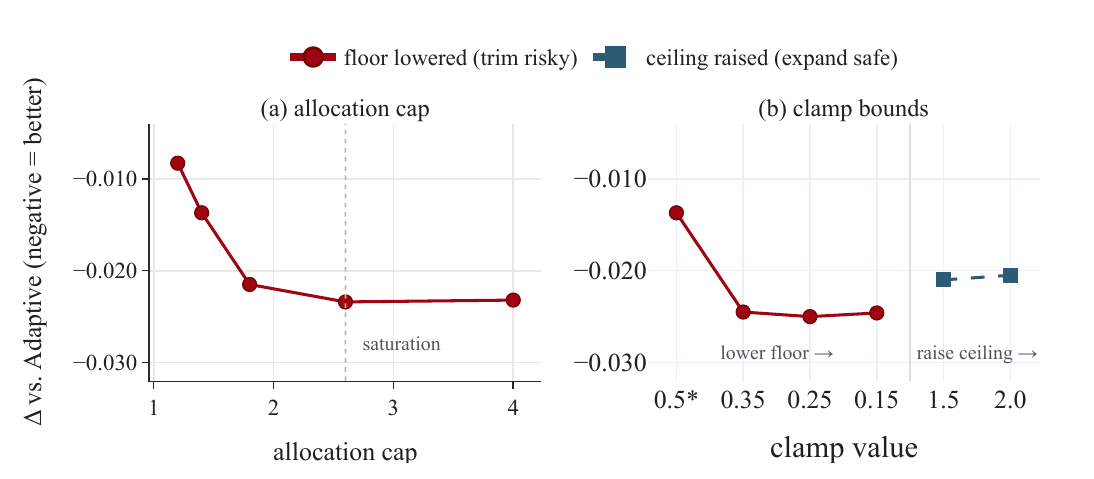}
  \caption{The two single-knob sweeps of Tables~\ref{app:cap}
  and~\ref{app:clamp}, $\Delta$ vs.\ FARO-Adaptive (negative is better). (a) Raising
  the allocation cap deepens the gain until the downward clamp binds. (b)
  Lowering the floor (red) deepens the gain, while raising the ceiling (blue)
  moves it less. Stars mark the default floor of $0.5$.}
  \label{fig:ablation}
\end{figure}

\begin{table}[t]
  \centering
  \caption{Scalar global-budget controllers under the same matched-pair
  protocol ($3$ seeds). Neither beats FARO-Adaptive, and neither wins a majority of
  seed pairs.}
  \label{app:negative}
  \small
  \begin{tabular}{lcc}
    \toprule
    Controller & $\Delta$ final & wins \\
    \midrule
    DS-FARO (difficulty-sensed) & \negDSFaro & \negDSFaroWins \\
    continuous fractional-Kelly & \negContKelly & \negContKellyWins \\
    \bottomrule
  \end{tabular}
\end{table}

\section{Dataset and splits}
\label{app:dataset}

\paragraph{Corpus.}
All experiments train on FineWeb token shards \citep{penedo2024fineweb} in the
packed \texttt{uint16} format used by the reference trainer
\citep{modded_nanogpt_2024}. The local corpus has $9$ training shards
($\approx9.0\times10^8$ tokens) and one validation shard
($\approx1.0\times10^8$ tokens).

\paragraph{Validation and probe slices.}
Validation loss uses a fixed prefix of the validation shard
($10{,}485{,}760$ tokens), held constant across runs. FARO's counterfactual probe
(used for routing, Eq.~\eqref{eq:return}) and the risk measurement read a
disjoint slice ($1{,}048{,}576$ tokens) immediately after that prefix, so no
reported validation number is computed on probe tokens.

\paragraph{Splits and fairness.}
The train/validation split is fixed. Every matched pair consumes the identical
data order and geometry/freeze timeline, and only the budget controller differs,
so the pair is a controlled comparison. No test split is used.

\paragraph{Corrupted regime.}
The ``scatter'' regime used in Table~\ref{tab:main} replaces a fixed fraction
($60\%$) of tokens at random positions, selected by a deterministic
per-microbatch hash. Difficulty is therefore not temporally clustered, which
stresses the controller under noisy risk. We use this regime as a stress test of
the controller, reported alongside the clean protocol rather than as a separate
benchmark.

\section{Hardware and software}
\label{app:hardware}

\paragraph{Hardware.}
All runs used NVIDIA H800 (80\,GB) GPUs, one GPU per run
(\texttt{torchrun} with \texttt{nproc\_per\_node}$=1$). Matched pairs were
sometimes executed concurrently on distinct GPUs; per-run wall-clock was
monitored so that GPU contention could not be mistaken for a method difference.

\paragraph{Software.}
Python 3.12, PyTorch 2.10 with CUDA 12.8, no distributed collectives inside the
training loop, and the mixed-precision settings of the reference trainer.

\paragraph{Cost.}
A single \expSteps{}-step run takes roughly $6$--$10$ minutes of H800 time. The
main table is \mainSeeds{} seeds $\times$ \mainReps{} repetitions $\times$ two
arms, and the ablation grids multiply this on the same single-GPU setting.

\paragraph{Reproducibility.}
Runs are seeded. The matched-pair design fixes the data order, the geometry
routing, and the freeze timeline across arms, so a pair differs only in the
budget controller, and per-step metrics and run configuration are written for
every run.

\section{Baselines}
\label{app:baseline}

\paragraph{Optimizer landscape.}
Table~\ref{tab:baselines} is not a head-to-head against the allocation axis; it
locates the budget study in the optimizer landscape. These optimizers were
trained to the end of the schedule and evaluated identically, so the comparison
is between whole optimizers under one recipe, not between budget controllers. The
result we draw on is that the FARO-Adaptive arm
(\faroPriorMean{} $\pm$ \faroPriorStd) matches NorMuonAndAdam
(\baseNorMuonAndAdam{} $\pm$ \baseNorMuonAndAdamStd) and clearly beats Muon and Adam: under one
framework, with no per-layer geometry hand-set, it reaches the level of the
strongest hand-designed recipe. FARO-ERC is the further step that improves on
this arm (Table~\ref{tab:main}).

\begin{table}[t]
  \centering
  \caption{Optimizer baselines on FineWeb at the end of the schedule, with the
  per-seed values behind each mean. Adam and Muon/NorMuonAndAdam use \baseAdamN{} and
  \baseMuonN{} seeds respectively; the FARO-Adaptive baseline is the same framework with the
budget controller off.
  (the matched-pair study against it, Table~\ref{tab:main}, uses \mainSeeds{}).
  Std is the population standard deviation.}
  \label{tab:baselines}
  \footnotesize
  \setlength{\tabcolsep}{4pt}
  \begin{tabular}{@{}llc@{}}
    \toprule
    Optimizer & seeds & mean $\pm$ std \\
    \midrule
    Adam    & $0,1,42,1024$ & $\baseAdam \pm \baseAdamStd$ \\
    & \multicolumn{2}{@{}l}{\footnotesize $\baseAdamPerSeed$} \\
    Muon    & $0,42,1024$   & $\baseMuon \pm \baseMuonStd$ \\
    & \multicolumn{2}{@{}l}{\footnotesize $\baseMuonPerSeed$} \\
    NorMuonAndAdam & $0,42,1024$ & $\baseNorMuonAndAdam \pm \baseNorMuonAndAdamStd$ \\
    & \multicolumn{2}{@{}l}{\footnotesize $\baseNorMuonAndAdamPerSeed$} \\
    FARO-Adaptive & $0,42,1024$ & $\faroPriorMean \pm \faroPriorStd$ \\
    & \multicolumn{2}{@{}l}{\footnotesize $\faroPriorPerSeed$} \\
    \bottomrule
  \end{tabular}
\end{table}

\paragraph{Why this comparison is not the paper's claim.}
The numbers in Table~\ref{tab:baselines} come from different training lengths and
seed counts than the main table, and involve different update rules. They are not
evidence that ERC beats a particular optimizer, and they are not a ranking of the
optimizers against each other. Each entry runs at its harness-default setting. The
variants also differ in more than the update rule: they share neither an update
cadence nor a weight-decay scaling. The table is therefore context, not a
tuning-matched comparison. The row our claim rests on is FARO-Adaptive, the same
framework with routing frozen and the budget controller off. It lands on
NorMuonAndAdam, which we therefore take as the human-prior reference.
The controlled evidence for ERC itself is the matched-pair comparison in
Table~\ref{tab:main}. This is also why we treat the matrix-optimizer family as
candidate arms inside FARO rather than as competitors. Improvements there change
the geometry, and our question is how to allocate among geometries once they are
fixed.

\paragraph{FARO-Adaptive (risk-neutral reference).}
The reference arm of the budget study is FARO-Adaptive: FARO with its
geometry-derived frozen budgets and no risk controller. It is the risk-neutral
point of the risk--reward spectrum, and every $\Delta$ in the main paper is
measured against it.

\paragraph{Scalar controllers.}
Two scalar controllers serve as negative controls, and both rescale a single
global budget. A fractional-Kelly rule multiplies the global budget by a factor
read from a probe return \citep{kelly1956}. DS-FARO multiplies it by the shortfall
of a held-out probe return below its trailing median. Both are evaluated as
matched pairs (Table~\ref{app:negative}); neither beats FARO-Adaptive, which is the
evidence that the useful signal is cross-block.

\paragraph{Reference trainer.}
The base model, data pipeline, and matrix updates follow the public
\texttt{modded-nanogpt} trainer \citep{modded_nanogpt_2024}; its schedule and
default hyperparameters are used unchanged for every part of training that is
not the budget controller.

\section{What the allocation does to the budgets}
\label{app:casestudy}

The informative case study for an optimizer is not an input--output example but
what the controller does to per-block budgets during a run.
Table~\ref{tab:erc-diag} summarizes the multipliers actually applied in a
representative run (seed $0$, clean data).

\begin{table}[t]
  \centering
  \caption{FARO-ERC allocation diagnostics for a representative run. The multiplier is
  the factor applied to a block's frozen budget: $1.0$ means unchanged, less than
  $1$ means shrunk.}
  \label{tab:erc-diag}
  \small
  \begin{tabular}{lc}
    \toprule
    Quantity & Value \\
    \midrule
    Blocks participating      & \blocksPerRun \\
    Multiplier range          & $\diagMultLo$--$\diagMultHi$ \\
    Fraction shrunk ($<1$)    & \diagShrunkPct\% \\
    Fraction expanded ($>1$)  & \diagExpandedPct\% \\
    Mean applied multiplier   & \diagMeanMult \\
    \bottomrule
  \end{tabular}
\end{table}

\paragraph{Reading the distribution.}
The allocation is budget-neutral by construction, since the capped solve
\eqref{eq:lambda_fn} preserves the total, yet the realized distribution of
multipliers is mildly asymmetric: at the default cap, \diagShrunkPct\% of blocks
are shrunk and \diagExpandedPct\% are expanded, giving a mean applied multiplier
of \diagMeanMult, i.e.\ a small net trim of the risky tail. The shrink side is the
stronger lever: lowering the floor, which lets the shrink go further, deepens the
gain more than raising the ceiling, which lets the expand go further
(Table~\ref{app:clamp}). The method therefore helps mostly by dampening the
riskiest block updates rather than by amplifying the safest ones.

\paragraph{Exactness of neutrality.}
The multiplier computation is exactly budget-neutral, verified to below
$10^{-6}$ in unit tests. The effective budgets recorded at runtime can deviate
slightly (around $10^{-2}$ relative) because the result is subsequently
re-clipped to the global routing bounds; the FARO-ERC-specific clamp exists to
relax that coupling, so that tightening either bound does not silently move the
other.

\paragraph{Behaviour under a corrupted stream.}
On the corrupted stream the realized per-block risks become noisier, so the
reallocation separates high-risk from low-risk blocks less cleanly and the margin
narrows (Table~\ref{tab:main}). The controller still lowers the mean loss and wins
\regimeScatterWins{} seeds: it degrades gracefully under noise rather than
breaking, which is the behaviour we want from a risk-driven rule.

\section{System overhead: the exploration phase is not free}
\label{app:overhead}

Automatic search costs something a fixed recipe does not. FARO spends an
exploration phase on shadow candidate updates, the held-out probe, and the
risk statistics, on top of compilation and warmup. Table~\ref{tab:overhead}
reports the full per-run system cost of the four arms on the same schedule.

\begin{table}[t]
  \centering
  \caption{Per-run system cost on the full \schedSteps{}-step schedule:
  throughput, training time (optimizer steps), wall time (incl.\ compilation and
  warmup), and peak memory.}
  \label{tab:overhead}
  \scriptsize
  \setlength{\tabcolsep}{3pt}
  \begin{tabular}{@{}lrrrrr@{}}
    \toprule
    System & tok/s & train & wall & alloc & res \\
    \midrule
    Adam            & \ohAdamTps & \ohAdamTrain & \ohAdamWall & \ohAdamAlloc & \ohAdamRes \\
    Muon            & \ohMuonTps & \ohMuonTrain & \ohMuonWall & \ohMuonAlloc & \ohMuonRes \\
    NorMuonAndAdam  & \ohNmTps & \ohNmTrain & \ohNmWall & \ohNmAlloc & \ohNmRes \\
    FARO-Adaptive   & \ohFaroTps & \ohFaroTrain & \ohFaroWall & \ohFaroAlloc & \ohFaroRes \\
    \bottomrule
    \multicolumn{6}{@{}l}{\scriptsize train/wall in s; alloc/res in MiB.} \\
  \end{tabular}
\end{table}

\paragraph{Against Muon.}
FARO-Adaptive is not slower than Muon: throughput is
\ohFaroTps{} vs \ohMuonTps{} tokens/s and training time
\ohFaroTrain{} vs \ohMuonTrain{}~s, so the extra machinery pays for itself in
step throughput. Its peak memory is higher (\ohFaroAlloc{} vs \ohMuonAlloc{}~MiB),
the price of holding the exploration state.

\paragraph{Against the human recipe.}
NorMuonAndAdam is the strongest hand-designed baseline. Against it,
FARO-Adaptive's average training time is \ohVsNmTrainPct\% higher
(\ohFaroTrain{} vs \ohNmTrain{}~s) and its peak allocated memory
\ohVsNmMemPct\% higher (\ohFaroAlloc{} vs \ohNmAlloc{}~MiB). Both sit in a range
we consider manageable. They also shrink as training stabilizes. Past the
exploration phase the steady-state step time and memory of the two arms are
close, so the gap is concentrated early rather than carried through the run.

\paragraph{Where the overhead comes from, and what bounds it.}
The additional cost is an exploration-phase cost. The reported peak is
\texttt{torch.cuda.max\_memory\_allocated()} and \texttt{reserved} over the whole
process, from compilation and warmup through the end of training, without a reset
before the first optimizer step. It therefore describes the peak the program has
to accommodate in practice, not the frozen steady-state footprint. The added
time is a small fraction of a full-scale run and is acceptable for the training
budgets we target, and the extra peak memory is a fixed resident term that can be
reduced by offloading the exploration state. We report both, rather than the
frozen steady state alone, so the cost of the automation is not hidden.

\section{A decomposition argument for routing and allocation}
\label{app:theory}

FARO makes two decisions per block: a geometry, chosen by measured return, and a
share of a fixed total budget, set by risk contribution. Neither decision is made
jointly over all blocks. For routing, this appendix bounds how far the per-block
choice can be from the joint one: the linearized return decomposes exactly, and
what the per-block choice ignores enters only at second order, through a residual
that the bound carries. For allocation, it shows that the budget coupling is
represented exactly rather than approximated, so the step is an exact solve of the
program it comes from. The two are different strengths of statement, and
Section~\ref{app:theory:joint} separates them. Neither argument claims that the
reduced decisions drive the original non-convex training run to a global optimum.

\subsection{Setup and standing assumptions}
\label{app:theory:setup}

We keep the notation of Section~\ref{sec:method}. Parameters are split into
$B=\blocksPerRun$ blocks $\theta=(\theta_1,\dots,\theta_B)$. A decision for block
$b$ is a pair $x_b=(\phi_b,\beta_b)$, a geometry and a budget $\beta_b\ge0$. The
step it induces is
\begin{equation}
  \begin{gathered}
  \Delta\theta_b(x_b)=\alpha\,\beta_b\,d_b(\phi_b),\\
  \Delta\theta(x)=\sum_{b=1}^{B}\Delta\theta_b(x_b),
  \end{gathered}
  \label{eq:th:step}
\end{equation}
with base learning rate $\alpha$ and unit direction $d_b(\phi)$. The realized
one-step loss increment on microbatch $\xi$ is
\begin{equation}
  Q_\xi(x)=\mathcal L(\theta+\Delta\theta(x);\xi)-\mathcal L(\theta;\xi).
  \label{eq:th:Q}
\end{equation}
Four assumptions hold throughout. They are mild, and each is checkable.
\begin{itemize}
  \item[\textbf{A1}] \emph{Disjoint supports.} Blocks act on non-overlapping
  parameter coordinates, so $\langle\Delta\theta_b,\Delta\theta_c\rangle=0$ for
  $b\ne c$.
  \item[\textbf{A2}] \emph{Frozen risk.} The per-block risk $r_b$ is measured from
  the previous step and held fixed while the allocation is solved, so the current
  gradient never sets the budget that acts on it.
  \item[\textbf{A3}] \emph{Separable quadratic risk.} Allocation scores a budget
  vector by the block-diagonal quadratic risk $R$ of \eqref{eq:th:R}, with each
  block's volatility read from its measured risk. Cross-block risk is not modeled.
  This is the model the method uses, not a claim that the true one-step objective is
  diagonal.
  \item[\textbf{A4}] \emph{Attainment.} Each local feasible set $\mathcal X_b$ is
  nonempty and compact, and the joint set is the product
  $\mathcal X=\prod_b\mathcal X_b$ when the guards are local.
\end{itemize}

\subsection{The linearized return decomposes exactly}
\label{app:theory:linear}

Let $g_{\xi,b}$ be the gradient block of block $b$ on microbatch $\xi$. By A1 the
first-order term separates with no further hypothesis: one backward pass through
the single global loss already delivers every $g_{\xi,b}$.

\begin{proposition}[Exact decomposition of the linearized return]
\label{prop:th:linear}
For every feasible $x$,
$-g_\xi^\top\Delta\theta(x)=\sum_b -g_{\xi,b}^\top\Delta\theta_b(x_b)$ identically,
and under A4
\begin{equation}
  \begin{gathered}
  \operatorname*{arg\,max}_{x\in\mathcal X}\,-g_\xi^\top\Delta\theta(x)\\
  =\prod_{b=1}^{B}\;\operatorname*{arg\,max}_{x_b\in\mathcal X_b}\,
   -g_{\xi,b}^\top\Delta\theta_b(x_b).
  \end{gathered}
  \label{eq:th:linear}
\end{equation}
The joint maximum equals the sum of the per-block maxima.
\end{proposition}

\begin{proof}
Take any $x_b^\star$ attaining the $b$-th block maximum. The product domain makes
$x^\star=(x_1^\star,\dots,x_B^\star)$ feasible, and per-block optimality summed
over $b$ gives $-g_\xi^\top\Delta\theta(x)\le-g_\xi^\top\Delta\theta(x^\star)$ for
every feasible $x$. Conversely, if a component of a joint maximizer were not a
block maximizer, replacing that component would raise the joint value.
\end{proof}

The router scores the probe return of Eq.~\eqref{eq:return} by finite difference,
so it never forms the inner product below. Write
$\rho_{b,\xi}(\phi)=-\langle g_{\xi,b},d_b(\phi)\rangle$ for the first-order model of
that return at the current parameters, with $g_{\xi,b}$ the gradient block of the
loss on microbatch $\xi$. The two agree to first order when the probe and the
training microbatch are drawn from the same distribution and the step is small,
which is the regime the argument below covers. Proposition~\ref{prop:th:linear} applies to
this first-order model, and through it to the routing decision as implemented.

What such a decision ignores first appears beyond the linear term: cross-block
curvature $d_b^\top H_{bc}\,d_c$, cross-block return covariance $\Sigma_{bc}$, the
function-space cross metric $F_{bc}$, and the joint tail of the loss increments.
The next subsection bounds the first three.

\subsection{A bounded cross-block residual}
\label{app:theory:assume}

Write the one-step objective as separable local utilities plus a residual,
\begin{equation}
  \begin{gathered}
  J(x)=\sum_{b=1}^{B}J_b(x_b)+R_t(x),\\
  \sup_{x\in\mathcal X}|R_t(x)|\le\varepsilon_t ,
  \end{gathered}
  \label{eq:th:H}
\end{equation}
where the local utility is the linearized return the router actually scores,
\begin{equation}
  J_b(\phi,\beta)=\rho_b(\phi)\,\beta ,
  \label{eq:th:Jb}
\end{equation}
so that at the reference budgets the per-block maximizer of $J_b$ is
$\arg\max_\phi\rho_b(\phi)$, which is the routing rule of
Section~\ref{sec:method:setup}. Routing is evaluated with the budget held at its
reference value $\bar\beta_b$, as in FARO-Adaptive, so during routing $x_b$ reduces
to the geometry $\phi_b$ and $\mathcal X_b$ to the geometry set; the budget only
becomes a free variable in the allocation step of Section~\ref{app:theory:erc}.
Everything the linearization drops sits in $R_t$. Its quadratic part
$R_t^{\mathrm{quad}}$ collects curvature, return variance, and the function-space
metric; its tail part collects the joint tail of the loss increments. If the
projected curvature, return covariance, and function-space metric split into
block-diagonal and cross parts,
$C_H=C_H^{\mathrm{blk}}+E_H$, $\Sigma=\Sigma^{\mathrm{blk}}+E_\Sigma$, and
$C_F=C_F^{\mathrm{blk}}+E_F$, then on a trust region $\|\beta\|_2\le\varrho$ the
quadratic part is bounded by the spectral norms of the cross parts,
\begin{equation}
  |R_t^{\mathrm{quad}}(x)|\le
  \tfrac{\varrho^2}{2}\bigl(\|E_H\|_2+\varsigma\|E_\Sigma\|_2
  +\nu\|E_F\|_2\bigr),
  \label{eq:th:eps}
\end{equation}
with weights $\varsigma\ge0$ and $\nu\ge0$ on the variance and function-space
terms. The tail part is not covered by \eqref{eq:th:eps} and is left inside
$\varepsilon_t$. Two points keep \eqref{eq:th:eps} honest.
First, the bound is informative only when the cross parts are small relative to the
block-diagonal ones; it does not make them small. Second, the trust radius $\varrho$
is not free: a fixed total budget pins $\|\beta\|_2\le\|\beta\|_1=\mathrm{Bud}$, so
shrinking $\varrho$ is not available here, and the cross terms must be controlled by
the block structure rather than by a small step. Whether they are is empirical, and
$E_H$, $E_\Sigma$, and $E_F$ are measurable: finite differences, retained
cross-covariance, and activation or logit drift respectively.

\subsection{Exact and approximate composition of the routing decision}
\label{app:theory:compose}

\begin{theorem}[Exact composition]
\label{thm:th:exact}
If $R_t\equiv0$ in \eqref{eq:th:H} and $\mathcal X=\prod_b\mathcal X_b$ satisfies
A4, then
\begin{equation}
  \operatorname*{arg\,max}_{x\in\mathcal X}J(x)
  =\prod_{b=1}^{B}\operatorname*{arg\,max}_{x_b\in\mathcal X_b}J_b(x_b).
\end{equation}
Every joint maximizer is a tuple of local maximizers, and conversely.
\end{theorem}

The proof repeats Proposition~\ref{prop:th:linear} with $J$ in place of the linear
term. The theorem is the ideal case; the next bound covers the residual.

\begin{theorem}[$2\varepsilon$ suboptimality]
\label{thm:th:approx}
Suppose \eqref{eq:th:H} holds on a common nonempty $\mathcal X$. Let $x^\star$
maximize $J$ and let $\hat x$ maximize the factorized objective
$J_{\mathrm{fac}}=\sum_b J_b$. Then
\begin{equation}
  J(x^\star)-J(\hat x)\le2\varepsilon_t .
  \label{eq:th:2eps}
\end{equation}
\end{theorem}

\begin{proof}
From \eqref{eq:th:H}, $J(x^\star)\le J_{\mathrm{fac}}(x^\star)+\varepsilon_t$.
Optimality of $\hat x$ for $J_{\mathrm{fac}}$ gives
$J_{\mathrm{fac}}(x^\star)\le J_{\mathrm{fac}}(\hat x)$, and \eqref{eq:th:H} again
gives $J_{\mathrm{fac}}(\hat x)\le J(\hat x)+\varepsilon_t$. Chaining the three
gives $J(x^\star)\le J(\hat x)+2\varepsilon_t$.
\end{proof}

\subsection{When routing is safe: a margin condition}
\label{app:theory:route}

\begin{proposition}[Route margin]
\label{prop:th:margin}
Let $\hat x$ maximize $J_{\mathrm{fac}}$, with joint geometry
$\hat{\boldsymbol\phi}$, and let
\begin{equation}
  m=\max_{x\in\mathcal X}J_{\mathrm{fac}}(x)
   -\max_{x\in\mathcal X:\,\boldsymbol\phi(x)\ne\hat{\boldsymbol\phi}}J_{\mathrm{fac}}(x)
\end{equation}
be the factorized margin over every other joint geometry. If $m>2\varepsilon_t$,
then every maximizer of the true objective $J$ carries the geometry
$\hat{\boldsymbol\phi}$.
\end{proposition}

\begin{proof}
For any $x$ with $\boldsymbol\phi(x)\ne\hat{\boldsymbol\phi}$, \eqref{eq:th:H}
gives $J(\hat x)\ge J_{\mathrm{fac}}(\hat x)-\varepsilon_t$ and
$J(x)\le J_{\mathrm{fac}}(x)+\varepsilon_t\le J_{\mathrm{fac}}(\hat x)-m+\varepsilon_t$.
Hence $J(\hat x)-J(x)\ge m-2\varepsilon_t>0$.
\end{proof}

Proposition~\ref{prop:th:margin} gives the freeze rule a meaning: a route may be
frozen when the incumbent--challenger margin covers the residual, not merely when
the ranking stops moving. Our implementation uses the settling window as a cheap
proxy for that margin.

\subsection{The allocation step is an exact dual decomposition}
\label{app:theory:erc}

Routing fixes the geometry of every block. Allocation then chooses the budget
vector $\beta$, and unlike routing it faces a constraint that couples all blocks at
once, the fixed total $\sum_b\beta_b=\mathrm{Bud}$ of
Eq.~\eqref{eq:budgetconstraint}. That constraint is not a product domain, so
A4 does not apply to it. The next result shows the coupling enters through exactly
one scalar, and that solving for that scalar recovers the exact optimum rather than
an approximation of it.

Under A2 and A3, take the allocation objective to be the separable quadratic risk
\begin{equation}
  R(\beta)=\sum_{b=1}^{B}\frac{\tilde r_b}{2\bar\beta_b}\,\beta_b^{2},
  \label{eq:th:R}
\end{equation}
read as the risk of a diagonal portfolio whose position in block $b$ is $\beta_b$,
whose volatility is the winsorized risk $\tilde r_b$ of Eq.~\eqref{eq:winsor}, and
whose reference weight is $\bar\beta_b$. FARO-ERC solves
\begin{equation}
  \begin{gathered}
  \min_{\beta}\;R(\beta)
  \quad\text{s.t.}\quad
  \sum_{b=1}^{B}\beta_b=\mathrm{Bud},\\
  \beta_b\in\bigl[\bar\beta_b/c,\;c\,\bar\beta_b\bigr],
  \end{gathered}
  \label{eq:th:primal}
\end{equation}
the box being the allocation cap of Eq.~\eqref{eq:erc}.

\begin{theorem}[ERC bisection solves \eqref{eq:th:primal} exactly]
\label{thm:th:erc}
Under A2 and A3, with $c>1$ so that the box in \eqref{eq:th:primal} is nonempty,
\eqref{eq:th:primal} has a unique solution, and it is
\begin{equation}
  \begin{gathered}
  \beta_b^\star=\bar\beta_b\,
  \operatorname{clip}\!\Big(\frac{\lambda^\star}{\tilde r_b},\,c^{-1},\,c\Big),\\
  \lambda^\star \text{ solves } \Lambda(\lambda^\star)=\mathrm{Bud},
  \end{gathered}
  \label{eq:th:ercsol}
\end{equation}
with $\Lambda$ as defined in Eq.~\eqref{eq:lambda_fn} and $\tilde r_b$ the
winsorized risk of Eq.~\eqref{eq:winsor}. Algorithm~\ref{alg:erc} therefore returns
the exact optimizer of \eqref{eq:th:primal}, not a heuristic approximation of it.
\end{theorem}

\begin{proof}
The objective \eqref{eq:th:R} is strictly convex on the box and the feasible set is
convex and compact, so the solution exists and is unique. Associate the multiplier
$\lambda$ with the total constraint and $\mu_b^\pm\ge0$ with the box. The
Lagrangian is separable across $b$, and stationarity in $\beta_b$ gives
$\tilde r_b\beta_b/\bar\beta_b=\lambda+\mu_b^+-\mu_b^-$. Solving the scalar problem
per block yields $\beta_b=\bar\beta_b\operatorname{clip}(\lambda/\tilde r_b,c^{-1},c)$:
interior blocks satisfy $\beta_b=\bar\beta_b\lambda/\tilde r_b$, and blocks at a
bound are pinned to it. The total constraint reduces to $\Lambda(\lambda)=\mathrm{Bud}$,
and $\Lambda$ is continuous and non-decreasing in $\lambda$ with
$\Lambda(0)=c^{-1}\mathrm{Bud}$ and $\Lambda(\lambda)\to c\,\mathrm{Bud}$ as
$\lambda\to\infty$. Since $c>1$, a root exists. $\Lambda$ is non-decreasing, so the
root may be non-unique when the box saturates every block, but $\beta^\star$ is
unique by strict convexity and every root returns it. Bisection recovers one.
\end{proof}

The theorem has three consequences that the paper relies on. First, the budget
coupling is carried by the single scalar $\lambda^\star$, so the allocation
decomposes into $B$ closed-form scalar solves plus one one-dimensional root; no
cross-block covariance is estimated. This is not a scalar controller. A scalar
controller multiplies every budget by one common factor, so it leaves the ratios
$\beta_b/\beta_c$ unchanged and cannot move mass between blocks. Here $\lambda^\star$
is a price attached to the total constraint, and each block responds to it through
its own risk, $\beta_b=\bar\beta_b\lambda^\star/\tilde r_b$; the ratios therefore
change, in the direction the risks dictate. Second, the optimum equalizes risk
contributions. Write $s_b(\beta):=\tilde r_b\beta_b/\bar\beta_b$ for block $b$'s
contribution under the model, so that $\partial_bR=s_b$. Every interior block is
assigned the same contribution $\lambda^\star$, and its budget is proportional to
$\bar\beta_b/\tilde r_b$; when the reference budgets coincide this is the classical
equal-risk-contribution condition $\tilde r_b\beta_b=\text{const}$. Third, the
contributions that are not equal are exactly the ones the cap pins: a block at its
lower bound contributes less than $\lambda^\star$ and one at its upper bound
contributes more. Equal risk contribution is thus derived from the KKT condition
rather than imposed.

Two features of the implementation follow from the same proof, and they act in
different places. Winsorizing $r_b$ into $\tilde r_b$ sets the coefficients of $R$,
so it defines the model rather than restricting it; it is what bounds the ratio
$\lambda^\star/\tilde r_b$ and hence how far one block's budget can move relative to
another's. The cap $c$ restricts the feasible box instead. Theorem~\ref{thm:th:erc}
is exact for the model and the box the method actually uses, so neither feature is
a correction applied after the fact. The optional ERC-specific clamp
$[\beta^{\mathrm{erc}}_{\min},\beta^{\mathrm{erc}}_{\max}]$ of
Section~\ref{sec:method:erc} tightens the box further and is exact for the same
program; when it binds, the applied budgets are re-clipped and neutrality holds only
to the tolerance measured in Appendix~\ref{app:casestudy}.

\subsection{What the two results cover, and what they do not}
\label{app:theory:joint}

The two decisions are covered at different strengths, and it is worth separating
them.

For routing, the statement is a bound on the loss. With budgets held at their
reference values, as in FARO-Adaptive, the decision is a geometry choice over a
product domain, so Theorem~\ref{thm:th:approx} applies directly: the routed decision
is $2\varepsilon_t$-suboptimal for the joint one-step objective, and
Proposition~\ref{prop:th:margin} says a route whose margin clears $2\varepsilon_t$
is not overturned. The objective here is the one-step loss increment that the
update induces, not the final validation loss; the two are linked by the training
dynamics, which this appendix does not model.

For allocation, the statement is exactness, not a loss bound.
Theorem~\ref{thm:th:erc} shows that the bisection returns the exact optimizer of
the risk program \eqref{eq:th:primal}, so the allocation step is not a heuristic and
the equal-risk-contribution property is a consequence rather than an assumption.
What the theorem does not do is identify the risk program with the loss: $R$ is a
risk potential, not $-J_{\mathrm{fac}}$, and no argument here bounds
$\max J_{\mathrm{fac}}$ over the simplex by $\min R$. So the composition of the two
steps does not inherit a $2\varepsilon_t$ loss bound. The evidence that minimizing
$R$ also lowers the final loss is the matched-pair comparison in
Section~\ref{sec:exp:main} plus the cap-by-LR ablation, and that is where we
rest it.

Two points follow about the budget constraint itself. The constraint
$\sum_b\beta_b=\mathrm{Bud}$ is a coupling, not something that disappears when the
total is fixed; fixing a total is what creates it. What Theorem~\ref{thm:th:erc}
establishes is that this coupling is represented exactly, by one price, rather than
approximated. Had we instead split the total into fixed per-block capacities $A_b$
with $\sum_bA_b\le\mathrm{Bud}$, the factorized feasible set would be a strict subset
of the coupled one and we would pay a pre-assignment loss that
\eqref{eq:th:eps} does not see; the exact dual avoids that particular loss. It does
not avoid the proxy question above, which is a different gap.

\subsection{Scope}
\label{app:theory:scope}

What the argument supports. The linearized return decomposes exactly, so scoring
each block on its own probe return is not an approximation at first order. Under a
bounded cross-block residual the per-block routing optimum is
$2\varepsilon_t$-suboptimal for the joint one-step objective, and a route whose
margin clears $2\varepsilon_t$ is not overturned by the coupling. The allocation
step is the exact solution of a separable risk program, with the budget constraint
carried by one scalar price and equal risk contributions following from it.

What the argument does not support. It gives no end-to-end convergence claim for
the non-convex run, and no claim that FARO-ERC beats a particular hand-designed
optimizer. It does not establish that $\varepsilon_t$ is small in our runs; that is
empirical, and the residual is measurable but we do not estimate it here. The
allocation result is exactness of the solve, not a loss bound: no step here relates
$\min R$ to $\max J_{\mathrm{fac}}$, so the two decisions do not compose into a
single $2\varepsilon_t$ statement, and the case for allocation lowering the loss
rests on the experiments. The risk signal is the measured update-to-weight ratio of
Eq.~\eqref{eq:uwr}, a proxy for volatility, not a modeled tail. A tail-aware variant
would need its own analysis: CVaR is subadditive, so local tail budgets do give a
conservative joint bound, but that bound is a different guarantee than the one
proved here, and we do not claim it.

\end{document}